\documentclass{article}

\usepackage{amsmath,amsthm,amssymb}
\usepackage{url}
\usepackage{color}
\usepackage{xspace}
\usepackage{dsfont}
\usepackage{algorithm}
\usepackage[noend]{algorithmic}

\newcommand{\remark}[1]{\textcolor{red}{\em #1}}

\usepackage{tikz}
\usepgflibrary{arrows}
\usetikzlibrary{arrows}

\renewcommand{\epsilon}{\varepsilon}
\newcommand{\R}{\mathds{R}}
\newcommand{\N}{\mathds{N}}

\newtheorem{theorem}{Theorem}
\newtheorem{definition}[theorem]{Definition}
\newtheorem{lemma}[theorem]{Lemma}

\newtheorem{corollary}[theorem]{Corollary}

\newcommand{\ignore}[1]{}

\newcommand{\oneoneea}{(1+1)~EA\xspace}
\newcommand{\oneoneeas}{(1+1)~EA*\xspace}

\newcommand{\wrt}{w.\,r.\,t.\xspace}

\newcommand{\ie}{i.\,e.\xspace}
\newcommand{\eg}{e.\,g.\xspace}
\newcommand{\LO}{\textsc{Leading\-Ones}\xspace}
\newcommand{\BV}{\textsc{BinVal}\xspace}
\newcommand{\ONEMAX}{\textsc{OneMax}\xspace}

\newcommand{\oneant}{1\nobreakdash-ANT\xspace}
\newcommand{\MMAS}{MMAS\xspace}
\newcommand{\MMASs}{MMAS*\xspace}

\newcommand{\tmin}{\tau_{\mathrm{min}}}
\newcommand{\set}[2]{\{#1 \; | \; #2\}}
\newcommand{\eqnComment}[2]{\underset{{\scriptstyle \text{#1}}}{#2}}

\usepackage[plainpages=false,pdfpagelabels]{hyperref}

\begin{document}

\title{Simple Max-Min Ant Systems and the Optimization of~Linear Pseudo-Boolean Functions}

\author{Timo K{\"o}tzing\\
Algorithms and Complexity\\
Max Planck Institute for Informatics\\
 Saarbr{\"u}cken, Germany
\and Frank Neumann\\
Algorithms and Complexity\\
Max Planck Institute for Informatics\\
 Saarbr{\"u}cken, Germany
\and Dirk Sudholt\\
International Computer Science Institute\\
 Berkeley, CA, USA
\and Markus Wagner\\
Algorithms and Complexity\\
Max Planck Institute for Informatics\\
 Saarbr{\"u}cken, Germany
} 

\maketitle

\begin{abstract}
With this paper, we contribute to the understanding of ant colony optimization (ACO) algorithms by formally analyzing their runtime behavior. We study simple MAX-MIN ant systems on
the class of linear pseudo-Boolean functions defined on binary strings of length $n$.
Our investigations point out how the progress according to function values is stored in pheromone.
We provide a general upper bound of $O((n^3 \log n)/ \rho)$ for two ACO variants on all linear functions, where $\rho$ determines the pheromone update strength.
Furthermore, we show improved  bounds for two well-known linear pseudo-Boolean functions called \ONEMAX and \BV and give additional insights using an experimental study.
\end{abstract}

\newpage

\section{Introduction}

Ant colony optimization (ACO) is an important class of stochastic search algorithms that has found many applications in combinatorial optimization as well as for stochastic and dynamic problems~\cite{DorigoStuetzleACOBook}.
The basic idea behind ACO is that ants construct new solutions for a given problem by carrying out random walks on a so-called construction graph. These random walks are influenced by the pheromone values that are stored along the edges of the graph. During the optimization process the pheromone values are updated according to good solutions found during the optimization which should lead to better solutions in further steps of the algorithm.

Building up a theoretical foundation of this kind of algorithms is a challenging task as these algorithms highly rely on random decisions. The construction of new solutions highly depends on the current pheromone situation in the used system which highly varies during the optimization run. Capturing the theoretical properties of the pheromone constellation is a hard task but very important to gain new theoretical insights into the optimization process of ACO algorithms.

With this paper, we contribute the theoretical understanding of ACO algorithms. Our goal is to gain new insights into the optimization process of these algorithms by studying them on the class of linear pseudo-Boolean functions. There are investigations of different depths on the behavior of simple evolutionary algorithms for this class of functions. The main result shows that each linear pseudo-Boolean function is optimization in expected time $O(n\log n)$ by the well known \oneoneea~\cite{DJWoneone,He2004,DBLP:conf/ppsn/Jagerskupper08}.

With respect to ACO algorithms, initial results on simplified versions of the MAX-MIN ant system \cite{Stutzle2000} have been obtained.
These studies deal with specific pseudo-Boolean functions defined on binary strings of length $n$. Such studies are primary focused on well-known linear example functions called \ONEMAX and \BV or the function \LO \cite{NeumannWittAlgorithmica09,Gutjahr2008a,DNSWLeading,DBLP:journals/swarm/NeumannSW09,NeumannSW10}.
Recently, some results for ACO algorithms on combinatorial optimization problems such as minimum spanning trees~\cite{NeumannW07} or the traveling salesman~\cite{AntsTsp10} have been obtained. These analyses assume that the pheromone bounds are attained in each iteration of the algorithms. This is the case if the pheromone update if a MAX-MIN ant system uses a strong pheromone update which forces the pheromone only to take on the maximum and minimum value. The analyses presented in \cite{NeumannW07}  and \cite{AntsTsp10} do not carry over to smaller pheromone updates. In particular, there are no corresponding polynomial upper bounds if the number of different function values is exponential with respect to the given input size.

We provide new insights into the optimization of MAX-MIN ant systems for smaller pheromone updates on functions that may attain exponentially many functions values.
Our study investigates simplified versions of the MAX-MIN ant system called \MMASs and \MMAS \cite{DBLP:journals/swarm/NeumannSW09} on linear pseudo-Boolean functions with non-zero weights. For these algorithms, general upper bounds of $O((n+ (\log n)/ \rho)D)$ and $O(((n^2 \log n) / \rho)D)$ respectively, have been provided for unimodal functions attaining $D$ different function values~\cite{DBLP:journals/swarm/NeumannSW09}. As linear pseudo-Boolean function are unimodal, these bounds carry over to this class of functions. However, they only give weak bounds for linear pseudo-Boolean functions attaining many function values (\eg for functions where the number of different function values is exponential in $n$).

We show an upper bound of $O((n^3 \log n)/ \rho)$ for \MMASs and \MMAS optimizing each linear pseudo-Boolean function. Furthermore, our studies show that the method of fitness-based partitions may also be used according to pheromone value, \ie, MAX-MIN ant systems store the progress made according to function values quickly into a pheromone potential. This is one of the key observations that we use for our more detailed analyses on \ONEMAX and \BV in which we improve the results presented in \cite{DBLP:journals/swarm/NeumannSW09}.

To provide further insights that are not captured by our theoretical analyses, we carry out an experimental study. Our experimental investigations give comparisons to simple evolutionary algorithms, and consider the impact of the chosen weights of the linear functions and pheromone update strength with respect to the optimization time. One key observation of these studies is that \ONEMAX is not the simplest linear function for the simple MAX-MIN ant systems under investigation. Additionally, the studies indicate that the runtime grows at most linearly with $1/\rho$ for a fixed value of $n$.

We proceed as follows. In Section~\ref{sec:smmas}, we introduce the simplified MAX-MIN ant systems that are subject to our investigations. In Section~\ref{sec:ACOOnLinearFunctions}, we provide general runtime bounds for the class of linear pseudo-Boolean functions and present specific results for \ONEMAX and \BV in Section~\ref{sec:impbounds}. Our experimental study which provides further insights is reported in Section~\ref{sec:exp}. Finally, we discuss our results and finish with some concluding remarks.

\section{Simplified MAX-MIN Ant Systems}
\label{sec:smmas}

We first describe the simplified MAX-MIN ant systems that will be investigated in the sequel.
The following construction graph is used to construct solutions for pseudo-Boolean optimization, \ie, bit strings of $n$ bits. It is based on a directed multigraph $C = (V, E)$.
In addition to a start node $v_0$, there is a node $v_i$ for every bit $i$, $1 \le i \le n$. This node can be reached from $v_{i-1}$ by two edges. The edge $e_{i,1}$ corresponds to setting bit $i$ to 1, while $e_{i,0}$ corresponds to setting bit $i$ to 0. The former edge is also called a \emph{1-edge}\index{1-edge}, the latter is called \emph{0-edge}\index{0-edge}. An example of a construction graph for $n=5$ is shown in Figure~\ref{fig:construction-multigraph}.

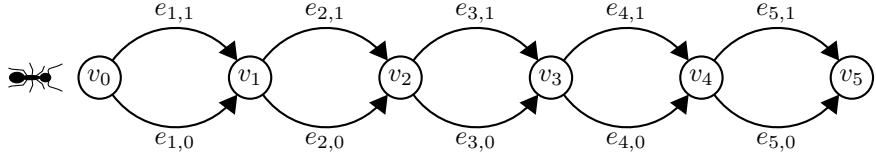
\begin{figure*}
\begin{center}
\begin{tikzpicture}[scale=2]
\tikzstyle{antbody}=[fill=black,draw=black];
\tikzstyle{antleg}=[black,rounded corners=0pt];
\begin{scope}[shift={(-0.55,0)},scale=0.2]
\draw[antleg] (6pt+5pt,0pt) -- ++(-5pt,+7pt) -- ++(-11pt,+1pt) -- ++(-3pt,+1pt);
\draw[antleg] (6pt+8pt,0pt) -- ++(+1pt,+7pt) -- ++(-3pt,+7pt);
\draw[antleg] (6pt+10pt,0pt) -- ++(+8pt,+6pt) -- ++(+2pt,+5pt);
\draw[antleg,rounded corners=1pt] (6pt-2pt+2*9pt-1pt+6pt+4pt,0pt) -- ++(-2pt,+8pt) -- ++(+10pt,+3pt) -- ++(+4pt,+2pt);
\draw[antleg] (6pt+5pt,0pt) -- ++(-5pt,-7pt) -- ++(-11pt,-1pt) -- ++(-3pt,-1pt);
\draw[antleg] (6pt+8pt,0pt) -- ++(+1pt,-7pt) -- ++(-3pt,-7pt);
\draw[antleg] (6pt+10pt,0pt) -- ++(+8pt,-6pt) -- ++(+2pt,-5pt);
\draw[antleg,rounded corners=1pt] (6pt-2pt+2*9pt-1pt+6pt+4pt,0pt) -- ++(-2pt,-8pt) -- ++(+10pt,-3pt) -- ++(+4pt,-2pt);
\filldraw[antbody] (6pt-2pt+9pt,0) ellipse (9pt and 2pt);
\filldraw[antbody] (0,0) ellipse (7pt and 5pt);
\filldraw[antbody] (6pt-2pt+2*9pt-1pt+6pt,0) ellipse (5pt and 4pt);
\end{scope}
\tikzstyle{node}=[black,fill=white,thick];
\tikzstyle{edge}=[black,thick,-triangle 60];
\foreach \x in {1,2,3,4,5} {
    \draw[edge,shorten >=8pt,shorten <=8pt] (\x-1,0) .. controls (\x-1+0.3,0.4) and (\x-1+0.7,0.4) .. (\x,0) node[pos=0.5,above] {$e_{\x,1}$};
    \draw[edge,shorten >=8pt,shorten <=8pt] (\x-1,0) .. controls (\x-1+0.3,-0.4) and (\x-1+0.7,-0.4) .. (\x,0) node[pos=0.5,below] {$e_{\x,0}$};
}
\foreach \x in {0,1,2,3,4,5} {
    \filldraw[node] (\x,0) circle (4pt) node {$v_{\x}$};
}
\end{tikzpicture}
\end{center}
\caption{Construction graph for pseudo-Boolean optimization with $n=5$ bits.}
\label{fig:construction-multigraph}
\end{figure*}

In a solution construction process an artificial ant sequentially traverses the nodes $v_0, v_1, \dots, v_n$.
The decision which edge to take is made according to pheromones on the edges. Formally, we denote pheromones by a function $\tau \colon E \to \R_0^+$. From $v_{i-1}$ the edge $e_{i,1}$ is then taken with probability $\tau(e_{i,1})/(\tau(e_{i,0}) + \tau(e_{i,1}))$.
In the case of our construction graph, we identify the path taken by the ant with a corresponding binary solution $x$
as described above and denote the path by $P(x)$.

All ACO algorithms considered here start with an equal
amount of pheromone on all edges: $\tau(e_{i, 0}) = \tau(e_{i,1}) = 1/2$.
Moreover, we ensure that ${\tau(e_{i, 0}) + \tau(e_{i,1}) = 1}$ holds, \ie,
pheromones for one bit always sum up to 1. This implies that the probability of taking a specific edge equals its pheromone value; in other words, pheromones and traversion probabilities coincide.

Given a solution $x$ and a path $P(x)$ of edges that have been chosen in the creation of~$x$, a pheromone update with respect to~$x$ is performed as follows. First, a $\rho$-fraction of all pheromones evaporates and a $(1-\rho)$-fraction remains. Next, some pheromone is added to edges that are part of the path $P(x)$ of $x$.
To prevent pheromones from dropping to arbitrarily small values, we follow the MAX-MIN ant system by St{\"u}tzle and Hoos~\cite{Stutzle2000} and restrict all pheromones to a bounded interval. The precise interval is chosen as $\left[1/n, 1- 1/n\right]$. This choice is inspired by standard mutations in evolutionary computation where for every bit an evolutionary algorithm has a probability of $1/n$ of reverting a wrong decision.

Depending on whether an edge $e$ is contained in the path~$P(x)$ of
the solution $x$, the pheromone values $\tau$ are updated to
$\tau'$ as follows:
\begin{equation*}
\label{eq:pheromone-update-formula}
\begin{array}{r@{}r@{}l@{}l}
 \tau'(e) =& \min\;& \displaystyle \left\{ (1-\rho)\cdot \tau(e) +
    \rho,\ 1- \frac{1}{n} \right\} & ~~\textnormal{if $e \in
P(x)$ and}\\
 \tau'(e) =& \max\;& \displaystyle \left\{ (1-\rho)\cdot
  \tau(e),\ \frac{1}{n} \right\} & ~~\textnormal{if $e \notin
P(x)$.}
\end{array}
\end{equation*}

The algorithm \MMAS now works as follows. It records the best solution found so far, known as best-so-far solution. It repeatedly constructs a new solution. This solution is then compared against the current best-so-far and it replaces the previous best-so-far if the objective value of the new solution is not worse. Finally, the pheromones are updated with respect to the best-so-far solution. A formal description is given in Algorithm~\ref{alg:MMAS}.

Note that when a worse solution is constructed then the old best-so-far solution is reinforced again. In case no improvement is found for some time, this means that the same solution $x^*$ is reinforced over and over again and the pheromones move towards the respective borders in~$x^*$. Previous studies~\cite{Gutjahr2008a} have shown that after $(\ln n)/\rho$ iterations of reinforcing the same solution all pheromones have reached their respective borders. This time is often called ``freezing time''~\cite{DBLP:journals/swarm/NeumannSW09}.

\renewcommand{\algorithmicloop}{\textbf{repeat forever}}
\begin{algorithm}[ht]
    \caption{\MMAS}
    \algsetup{indent=1.5em}
    \begin{algorithmic}[1]
        \STATE Set $\tau_{(u,v)} = 1/2$ for all $(u,v) \in E$.
        \STATE Construct a solution $x^*$.
        \STATE Update pheromones \wrt{} $x^*$.
        \LOOP
            \STATE Construct a solution $x$.
            \STATE \textbf{if} $f(x) \ge f(x^*)$ \textbf{then} $x^*:=x$.
            \STATE Update pheromones \wrt{} $x^*$.
        \ENDLOOP
    \end{algorithmic}
    \label{alg:MMAS}
\end{algorithm}

We also consider a variant of \MMAS known as \MMASs~\cite{DBLP:journals/swarm/NeumannSW09} (see Algorithm~\ref{alg:MMASs}). The only difference is that the best-so-far solution is only changed in case the new solution is strictly better. In application of ACO often this kind of strategy is used. However, in~\cite{DBLP:journals/swarm/NeumannSW09} it was argued that \MMAS works better on functions with plateaus as \MMAS is able to perform a random walk on equally fit solutions.

\begin{algorithm}[ht]
    \caption{\MMASs}
    \algsetup{indent=1.5em}
    \begin{algorithmic}[1]
        \STATE Set $\tau_{(u,v)} = 1/2$ for all $(u,v) \in E$.
        \STATE Construct a solution $x^*$.
        \STATE Update pheromones \wrt{} $x^*$.
        \LOOP
            \STATE Construct a solution $x$.
            \STATE \textbf{if} $f(x) > f(x^*)$ \textbf{then} $x^*:=x$.
            \STATE Update pheromones \wrt{} $x^*$.
        \ENDLOOP
    \end{algorithmic}
    \label{alg:MMASs}
\end{algorithm}

In the following we analyze the performance of \MMAS and \MMASs on linear functions. We are interested in the number of iterations until the first global optimum is found. This time is commonly called optimization time.

Note that the pheromones on the 1-edges $e_{\cdot,1}$ suffice to describe all pheromones as $\tau(e_{i, 0}) + \tau(e_{i, 1}) = 1$. When speaking of pheromones, we therefore often focus on pheromones on 1-edges. 

\section{General Results}
\label{sec:ACOOnLinearFunctions}

We first derive general upper bounds on the expected optimization time of \MMAS and \MMASs on linear pseudo-Boolean functions. A linear pseudo-Boolean function for an input vector $x=(x_1,\ldots,x_n)$ is a function $f\colon\left\{0,1\right\}^{n}\mapsto\mathbb{R}$, with $f(x)=\sum^{n}_{i=1}w_{i}x_{i}$ and \emph{weights} $w_{i}\in\mathbb{R}$. 

We only consider positive weights since a function with a negative weight $w_i$ may be transformed into a function with a positive weight $w_i' = -w_i$ by exchanging the meaning of bit values $0$ and $1$ for bit~$i$. This results in a function whose value is by an additive term of $w_i'$ larger. This and exchanging bit values does not impact the behavior of our algorithms.
We also exclude weights of 0 as these bits do not contribute to the fitness.
Finally, in this section we assume without loss of generality that the weights are ordered according to their values: $w_{1} \geq w_{2} \geq \ldots \geq w_{n}$.

Two well-known linear functions are the function \ONEMAX where $w_1 = w_2 = \dots = w_n = 1$ and the function \BV where $w_i = 2^{n-i}$. These functions represent two extremes: for \ONEMAX all bits are of equal importance, while in \BV a bit at position~$i$ can dominate all bits at positions $i+1, \dots, n$.

\subsection{Analysis Using Fitness-based Partitions}
\label{sec:FitnessBasedPartitions}

We exploit a similarity between \MMAS, \MMASs and evolutionary algorithms to obtain a first upper bound.
We use the method of fitness-based partitions, also called fitness-level method, to estimate the expected optimization time. This method has originally been introduced for the analysis of elitist evolutionary algorithms (see, \eg, Wegener~\cite{Wegener2002}) where the fitness of the current search point can never decrease. The idea is to partition the search space into sets $A_1, \dots, A_m$ that are ordered with respect to fitness. Formally, we require that for all $1 \le i \le m-1$ all search points in $A_i$ have a strictly lower fitness than all search points in $A_{i+1}$. In addition, $A_m$ must contain all global optima.

Now, if $s_i$ is (a lower bound on) the probability of discovering a new search point in $A_{i+1} \cup \dots \cup A_m$, given that the current best solution is in $A_i$, the expected optimization time is bounded by $\sum_{i=1}^{m-1} 1/s_i$ as $1/s_i$ is (an upper bound on) the expected time until fitness level~$i$ is left and each fitness level has to be left at most once.

Gutjahr and Sebastiani~\cite{Gutjahr2008a} as well as Neumann, Sudholt, and Witt~\cite{DBLP:journals/swarm/NeumannSW09} have adapted this method for \MMASs. If the algorithm does not find a better search point for some time, the same solution $x^*$ is reinforced over and over again, until eventually all pheromones attain their borders corresponding to the bit values in $x^*$. We say that then all pheromones are saturated. In this setting the solution creation process of \MMASs equals a standard bit mutation of $x^*$ in an evolutionary algorithm. If $s_i$ is (a lower bound on) the probability that the a mutation of $x^*$ creates a search point in $A_{i+1} \cup \dots \cup A_m$, then the expected time until \MMASs leaves fitness level $A_i$ is bounded by $(\ln n)/\rho + 1/s_i$ as either the algorithm manages to find an improvement before the pheromones saturate or the pheromones saturate and the probability of finding an improvement is at least $s_i$.
This results in an upper bound of $m \cdot (\ln n)/\rho + \sum_{i=1}^{m-1} 1/s_i$ for \MMASs.

One restriction of this method is, however, that fitness levels are only allowed to contain a single fitness value; in the above bound $m$ must equal the number of different fitness values. Without this condition---when a fitness level contains multiple fitness values---\MMASs may repeatedly exchange the current best-so-far solution within a fitness level. This can prevent the pheromones from saturating, so that the above argument breaks down. For this reason, all upper bounds in~\cite{DBLP:journals/swarm/NeumannSW09} grow at least linearly in the number of function values.

The following lemma gives an explanation for the time bound $(\ln n)/\rho$ for saturating pheromones. This time is also called \emph{freezing time}. We present a formulation that holds for arbitrary sets of bits. Though we do not make use of the larger generality, this lemma may be of independent interest.

\begin{lemma}\label{lem:freezing}
Given an index set $I \subseteq \{1, \dots, n\}$ we say that a bit is in $I$ if its index is in $I$.
Let $x^*$ be the current best-so-far solution of \MMAS or \MMASs optimizing an arbitrary function.
After $(\ln n)/\rho$ further iterations either all pheromones corresponding to bits in~$I$ have reached their respective bounds in $\{1/n, 1-1/n\}$ or $x^*$ has been replaced by some search point $x^{**}$ with $f(x^{**}) \ge f(x^*)$ for \MMAS and $f(x^{**}) > f(x^*)$ for \MMASs such that $x^{**}$ differs from $x^*$ in at least one bit in $I$.
\end{lemma}

\begin{proof}
Assume the bit values of the bits in $I$ remain fixed in the current best-so-far solution for $(\ln n)/\rho$ iterations as otherwise there is nothing to prove. In this case for every bit $x^*_i$ with $i \in I$ the same bit value $x^*_i$ has been reinforced for $(\ln n)/\rho$ iterations. This implies that for the edge in the construction graph representing the opposite bit value the lower pheromone border $1/n$ has been reached, as for any initial pheromone $0 \le \tau_i \le 1$ on this edge we have $(1-\rho)^{(\ln n)/\rho} \tau_i \le e^{-\ln n} \tau_i \le 1/n$.
\end{proof}

So far, the best known general upper bounds for \MMAS and \MMASs that apply to every linear function are $O((n^2 \log n)/\rho \cdot 2^n)$ and $O((\log n)/\rho \cdot 2^n)$, respectively, following from upper bounds for unimodal functions~\cite{DBLP:journals/swarm/NeumannSW09}. The term $2^n$ results from the fact that in the worst case a linear function has $2^n$ different function values. This is the case, for instance, for the function \BV. An exponential upper bound for linear functions is, of course, unsatisfactory. The following theorem establishes a polynomial upper bound (with respect to $n$ and $1/\rho$) for both algorithms.
\begin{theorem}
\label{the:general-upper-bound-linear}
The expected optimization time of \MMAS and \MMASs on every linear function is in $O((n^3 \log n)/\rho)$.
\end{theorem}

\begin{proof}
The proof is an application of the above-described fitness-based partitions method.
In the first step, we consider the time needed to sample a solution which is at least on the next higher fitness level. We analyze the two situations when the pheromones are either saturated  
or not.
Our upper bound is the result of the repetition of such advancing steps between fitness levels.

We define the fitness levels
$$A_{i} = \left\{ x \in \left\{0,1\right\}^{n} \left| \; \sum^{i}_{j=1}w_{j} \leq f(x) < \sum^{i+1}_{j=1}w_{j} \right\} \right.$$
and apply the fitness-based partitions method. Intuitively, the fitness levels are defined such that a solution $x$ is at least on fitness level $A_{i}$ if the leftmost bit of $x$ with value 0 is at position $i+1$.

The expected time spent sampling solutions on fitness level $A_i$ (\ie, without sampling a solution of a higher fitness level) is the combination of the time spent in $A_i$ with saturated pheromone values and the time spent in $A_i$ with unsaturated pheromone values.
In the following, we analyze for both situations the probabilities to sample a solution of a higher fitness level.
In the end, as \MMAS and \MMASs might not remain in one of both situations exclusively, but alternates between situations of saturated and unsaturated pheromone values, we take the sum of both runtimes as an upper bound.

First, when the pheromone values are saturated, the probability to flip the leftmost zero and keep all other bits is
$1/n \cdot (1-1/n)^{n-1} \geq 1/n \cdot 1/e$,
as  $(1-1/n)^{n-1} \geq 1/e \geq (1-1/n)^n$ holds for all $n \in \N$.
This results in a probability of $\Omega(1/n)$ of advancing in such a situation. Thus, the algorithm stays an expected number of $O(n)$ steps with saturated pheromone values without sampling a solution on a higher fitness level.

For the second case, when the pheromone values are not saturated, let $i < n$, and suppose $x^* \in A_i$ is the current best solution. Then, let us denote by $G = \bigcup_{j>i} A_j$ all \emph{good} solutions that are at least on the next higher fitness level, and by $B = \left\{ x \in \left\{0,1\right\}^n | f(x) \geq f(x^*), x \notin G \right\}$ all \emph{bad} solutions that are in $A_i$ with an equal or a higher function value than $x^*$. Thus, every improving sampled solution belongs to $G \cup B$.

Let $h:\left\{0,1\right\}^{n}\mapsto\left\{0,1\right\}^{n}$ be the function that returns for a given solution $x$ a solution $x' \in G$, where the leftmost 0 was flipped to 1.

Let $P(x)$ be the probability to sample a new solution $x$. Then the probability $q$ to sample any of $x$ and $h(x)$ is greater than or equal to $P(x)$. The probability to sample $h(x)$ is the probability to sample any of $x$ and $h(x)$ times the probability that the leftmost zero of $x$ was sampled as a one. Thus, for all $x$, $P(h(x)) = q \cdot 1/n \geq P(x)/n$, as the pheromone values are at least $1/n$.

Furthermore, each solution $h(x)$ has at most $n$ preimages with respect to $h$. Note that, for all $x \in B$, $h(x) \in G$.

Thus, the probability to sample the next solution $x \in G$ is
\begin{align*}
P(x \in G)
& = \sum_{x \in G} P(x)
\geq \frac{\sum_{x \in B} P(h(x))}{n}\\
& \geq \frac{\sum_{x \in B} P(x)}{n^2}
= \frac{P(x \in B)}{n^2}.
\end{align*}

So, sampling a good solution is at least $1/n^2$ times as likely as sampling a bad one.

Furthermore, while no bad solutions are sampled, at most $(\ln n)/ \rho$ steps are spent before the pheromone values are saturated (based on Lemma~\ref{lem:freezing}).

Thus, up to $(\ln n) / \rho$ steps are spent with unsaturated pheromone values before sampling a new solution, and $O((n^2 \log n)/\rho)$ steps are spent in total sampling solutions in $B$ before sampling a solution in $G$.

Consequently, the time spent on one fitness level is the sum of the times spent in either situation of the pheromone values, that is, $O(n) + O((n^2 \log n) / \rho) = O((n^2 \log n) / \rho)$.

Finally, as there are $n$ fitness levels, the afore-described steps have to be performed at most $n$-times, which yields a total runtime of $O((n^3 \log n) / \rho)$.
\end{proof}

\subsection{Fitness-based Partitions for Pheromones}
\label{sec:PheromonePartitions}

We describe an approach for extending the argument on $f$-based partitions to pheromones instead of the best-so-far solution.
This alternate approach is based on \emph{weighted pheromone sums}. Given a vector of pheromones $\tau$ and a linear function $f$, the \emph{weighted pheromone sum (wps)} is $f(\tau)$.

The idea is that, during a run of the algorithm, the wps should rise until \emph{saturated} with respect to the current best search point, and then a significant improvement should have a decent probability.

Define a function $v$ on bit strings as follows.
$$
v(x) = \sum_{i=1}^n
\begin{cases}
(1-\frac{1}{n})w_i,		&\mbox{if }x_i=1;\\
\frac{1}{n} w_i,			&\mbox{otherwise.}
\end{cases}
$$

A pheromone vector $\tau$ is called \emph{saturated} with respect to a search point $x^*$ iff
$$
f(\tau) \geq v(x^*).
$$

Note that this definition of saturation is very much different from previous notions of saturation in the literature.

Let, for all $i$, $a_i$ be the bit string starting with $i$ $1$s and then having only $0$s.

We let
$$A_i = \{ x \; \mid \; f(a_i) \leq f(x) < f(a_{i+1}) \}$$
and
$$B_i = \{ \tau \; \mid \; v(a_i) \leq f(\tau) < v(a_{i+1}) \}.$$

While $(A_i)_i$ captures the progress of the search points towards the optimum, $(B_i)_i$ captures the progress of the pheromones.

\newcommand{\capt}[1]{\tau^{\mathrm{cap},#1}}

\begin{lemma}\label{lem:SaturatePheromones}
For all $i$, if the best-so-far solution was in $\bigcup_{j \geq i} A_j$ for at least $(\ln n)/\rho$ iterations, then $\tau \in B_i$.
\end{lemma}
\begin{proof}
Let $h$ be such that $\forall s: h(s) = s(1-\rho) + \rho$.
Let $\tau^0$ be the vector of pheromones when the algorithm samples a solution in $\bigcup_{j \geq i} A_j$ for the first time, and let $(\tau^t)_{t}$ be the vectors of pheromones in the successive rounds.
Further, we define the sequence of \emph{capped-out pheromones} $(\capt{t})_{t}$ such that, for all $t$, $\capt{t}_j = \min(h^t(1/n),\tau^t_j)$.
For this capped-out version of pheromones we have, for all $t$ with $h^{t+1}(1/n) \leq 1-1/n$,
$$f(\capt{t+1}) \geq (1-\rho)f(\capt{t}) + \rho f(a_i),$$
as pheromones will evaporate and at least an $f(a_i)$ weighted part of them will receive new pheromone $\rho$ (note that the old capped-out pheromone raised by $h$ cannot exceed the new cap).
Thus, we get inductively
$$\forall t: f(\capt{t}) \geq h^t(1/n)f(a_i).$$
As we know from Lemma~\ref{lem:freezing}, for $t \geq (\ln n)/\rho$ we have $h^t(1/n) \geq 1-1/n$, and, thus, $\tau^t \in B_i$.
\end{proof}

This argument opens up new possibilities for analyses and we believe it to be of independent interest.

If it is possible to show, for all $i$, if $\tau \in B_i$, then the probability of sampling a new solution in $\bigcup_{j > i} A_j$ is $\Omega(1/n)$, then Lemma~\ref{lem:SaturatePheromones} would immediately improve the bound in Theorem~\ref{the:general-upper-bound-linear} to $O\left(n^2 + (n \log n)/\rho\right)$.
However, this claim is not easy to show.

It is possible to prove this for the special case of \ONEMAX using the following theorem by Gleser~\cite{Gleser:j:75}. This theorem gives a very nice handle on estimating probabilities for sampling above-average solutions for \ONEMAX.

\begin{theorem}[Gleser~\cite{Gleser:j:75}]\label{thm:Gleser}
Let $\tau,\tau'$ be two pheromone vectors such that, for all $j\leq n$, the sum of the $j$ least values of $\tau$ is at least the sum of the $j$ least values of $\tau'$. Let $\lambda$ be the sum of the elements of $\tau'$. Then it is at least as likely to sample $\lfloor \lambda + 1 \rfloor$ $1$s with $\tau$ as it is with $\tau'$.
\end{theorem}

We can use this theorem to get a good bound for \ONEMAX: the worst case for the probability of an improvement is attained when all but at most one pheromones are at their respective borders. In this situation, when there are still $i$ $1$s missing, the probability of an improvement is $\Omega(i/n)$.
Combining this with Lemma~\ref{lem:SaturatePheromones}, we get the following result.

\begin{corollary}
\label{cor:Onemax}
The expected optimization time of \MMAS and \MMASs on \ONEMAX is bounded by \[
\mathord{O}\mathord{\left(\sum_{i=1}^n \frac{n}{i} + n \cdot (\ln n)/\rho\right)} = \mathord{O}\mathord{\left((n \log n)/\rho\right)}.
\]
\end{corollary}
This re-proves the $O((n \log n)/\rho)$-bound for \MMASs in~\cite{DBLP:journals/swarm/NeumannSW09} and it improves the current best known bound for \MMAS on \ONEMAX by a factor of~$n^2$. We will present an even improved bound for \ONEMAX in Section~\ref{sec:ACOOnOneMax}.

However, with regard to general linear functions it is not clear how this argument can be generalized to arbitrary weights.
In the following section we therefore turn to the investigation of concrete linear functions.

\section{Improved Bounds for Selected Linear Functions}
\label{sec:impbounds}
Using insights from Section~\ref{sec:PheromonePartitions} we next present improved upper bounds for the function \ONEMAX (Section~\ref{sec:ACOOnOneMax}). 
Afterwards, we focus on the special function \BV (Section~\ref{sec:ACOOnBinVal}).
These resulting bounds are much stronger than the general ones given in Section~\ref{sec:ACOOnLinearFunctions} above.

\subsection{OneMax}
\label{sec:ACOOnOneMax}

Recall the bound $O((n \log n)/\rho)$ for \MMAS and \MMASs from Corollary~\ref{cor:Onemax}. In the following we prove a bound of $O(n \log n + n/\rho)$ for both \MMAS and \MMASs by more detailed investigations on the pheromones and their dynamic growth over time. This shows in particular that the term $1/\rho$ has at most a linear impact on the total expected optimization time.

Let $v$ and $a_i$ be as in Section~\ref{sec:PheromonePartitions}. For all $i \leq n$, we let $\alpha_i = i(1-1/n)$. Observe that $v(a_i) = \alpha_i + (n-i)/n$ and, in particular, $\alpha_i \le v(a_i) \le \alpha_i + 1$.

The following lemma gives a lower bound on the pheromone $f(\tau^+)$ after on iteration. Note that for \ONEMAX $f(\tau^+)$ corresponds to the sum of pheromones.
\begin{lemma}\label{lem:OneStepPheromoneDrift}
Let $i<j$ and let $\tau$ be the current pheromones with $v(a_i) \leq f(\tau) < v(a_{i+1})$ and suppose that the best-so-far solution has at least $j$ ones. We denote by $\tau^+$ the pheromones after one iteration of \MMAS or \MMASs. Then we have $f(\tau^+) \geq v(a_{i+1})$ or
\begin{equation}\label{eq:OneStepGain}
f(\tau^{+}) - \alpha_i \geq (f(\tau)-\alpha_i)(1-\rho) + (j-i)\rho \geq v(a_i).
\end{equation}
\end{lemma}
\begin{proof}
Suppose that in rewarding bit positions from $\tau$ to get $\tau^+$, exactly $k$ positions cap out at the upper pheromone border $1-1/n$. From $f(\tau^+) < v(a_{i+1})$ (otherwise there is nothing left to show), we have $k \leq i$. We decompose $f(\tau^{+})$ into the contribution of the capped-out bits (which is $\alpha_k$, being $k$ of the $j$ rewarded positions) and the rest. Then we have
$$f(\tau^{+}) \geq \alpha_k + (f(\tau)-\alpha_k)(1-\rho) + (j-k)\rho.$$
We now get
\begin{align*}
&\alpha_i + (f(\tau)-\alpha_i)(1-\rho) + (j-i)\rho\\
= & \; \alpha_k + \alpha_{i-k} + (f(\tau)-\alpha_k - \alpha_{i-k})(1-\rho) + (j-k)\rho + (k-i)\rho\\
\leq & \; f(\tau^+) + \alpha_{i-k} - \alpha_{i-k}(1-\rho) + (k-i)\rho\\
= & \; f(\tau^+) + \rho(\alpha_{i-k} - (i-k))\\
\leq & \; f(\tau^+).
\end{align*}
From $v(a_{i}) \leq f(\tau)$ we get
\begin{align*}
f(\tau^+)
& \geq \alpha_i + (f(\tau)-\alpha_i)(1-\rho) + (j-i)\rho\\
& \geq \alpha_i + (v(a_i)-\alpha_i)(1-\rho) + (j-i)\rho\\
& = v(a_i) + \rho (j-i - v(a_i)+\alpha_i).
\end{align*}
From $v(a_i)-\alpha_i < 1$ and $j>i$ we get the desired conclusion.
\end{proof}
One important conclusion from Lemma~\ref{lem:OneStepPheromoneDrift} is that once the sum of pheromones is above some value $v(a_i)$, it can never decrease below this term. 

Now we extend Lemma~\ref{lem:OneStepPheromoneDrift} towards multiple iterations. The following lemma shows that, unless a value of $v(a_{i+1})$ is reached, the sum of pheromones quickly converges to $\alpha_j$ when $j$ is the number of ones in the best-so-far solution.
\begin{lemma}\label{lem:MultiStepPheromoneDrift}
Let $i<j$ and let $\tau$ be the current pheromones with $v(a_i) \leq f(\tau) < v(a_{i+1})$ and suppose that the best-so-far solution has at least $j$ ones. For all $t$, we denote by $\tau^{t}$ the pheromones after $t$ iterations of \MMAS or \MMASs. Then we have for all $t$ $f(\tau^{t}) \geq v(a_{i+1})$ or
$$
f(\tau^t)-\alpha_i \geq (j-i)(1-(1-\rho)^t).
$$
\end{lemma}
\begin{proof}
Inductively for all $t$, we get from Lemma~\ref{lem:OneStepPheromoneDrift} $f(\tau^{t}) \geq v(a_{i+1})$ or
\begin{align*}
f(\tau^t)-\alpha_i
 & \geq (f(\tau^0) - \alpha_i)(1-\rho)^t + (j-i)\rho\sum_{i=0}^{t-1}(1-\rho)^i\\
 & = (f(\tau^0) - \alpha_i)(1-\rho)^t + (j-i)\rho\frac{1-(1-\rho)^{t}}{1-(1-\rho)}\\
 & = (f(\tau^0) - \alpha_i)(1-\rho)^t + (j-i)(1-(1-\rho)^{t})\\
 & \geq (j-i) (1-(1-\rho)^t).
\end{align*}
\end{proof}

\begin{theorem}
The expected optimization time of \MMAS and \MMASs on \ONEMAX is $O(n \log n + n/\rho)$.
\end{theorem}
\begin{proof}
Define $v(a_{-1}) = 0$. Let $\tau$ be the current pheromones and $\tau^t$ be the pheromones after $t$ iterations of \MMAS or \MMASs.
We divide a run of \MMAS or \MMASs into phases: the algorithm is in Phase~$j$ if $f(\tau) \ge f(a_{j-1})$, the current best-so-far solution contains at least $j$ ones, and the conditions for Phase~$j+1$ are not yet fulfilled.
We estimate the expected time until each phase is completed, resulting in an upper bound on the expected optimization time.

We first deal with the last $n/2$ phases and consider some Phase~$j$ with $j \ge n/2$.
By Lemma~\ref{lem:MultiStepPheromoneDrift} after $t$ iterations we either have $f(\tau^t) \ge v(a_{j+1})$ or $f(\tau^t) \ge \alpha_{j-1} + 1-(1-\rho)^t$. Setting $t := \lceil 1/\rho \rceil$, this implies $f(\tau^t) \ge \alpha_{j-1} + 1-e^{-\rho t} \ge \alpha_{j-1} + 1-1/e$. We claim that then the probability of creating $j+1$ ones is $\Omega((n-j)/n)$.

Using $j \ge n/2$, the total pheromone $\alpha_{j-1} + 1-1/e$ can be distributed on an artificially constructed pheromone vector $\tau'$ as follows. We assign value $1-1/n$ to $j-1$ entries and value $1/n$ to $n-j$ entries. As $(n-j)/n \le 1/2$, we have used pheromone of $\alpha_{j-1} + 1/2$ and so pheromone $1-1/2-1/e = \Omega(1)$ remains for the last entry.
We now use Theorem~\ref{thm:Gleser} to see that it as least as likely to sample a solution with $j+1$ ones with the real pheromone vector $\tau^t$ as it is with $\tau'$.
By construction of $\tau'$ this probability is at least $(1-1/n)^{j-1} \cdot (1/2-1/e) \cdot (n-j)/n \cdot (1-1/n)^{n-j-1} \ge (n-j)(1/2-1/e)/(en)$ as a sufficient condition is setting all bits with pheromone larger than $1/n$ in $\tau'$ to 1 and adding exactly one 1-bit out of the remaining $n-j$ bits.

Invoking Lemma~\ref{lem:MultiStepPheromoneDrift} again for at least $j+1$ ones in the best-so-far solution, we get
$f(\tau^{t+t'}) \ge \alpha_{j-1} + 2 (1-(1-\rho)^{t'})$,
which for $t' := \lceil 2/\rho \rceil$ yields $f(\tau^{t+t'}) \ge \alpha_{j-1} + 3/2 \ge \alpha_j + 1/2 \ge v(a_j)$ as $j \ge n/2$.

For the phases with index $j < n/2$ we construct a pessimistic pheromone vector $\tau'$ in a similar fashion. We assign value $1-1/n$ to $j-2$ entries, value $1/n$ to $n-j$ entries, and put the remaining pheromone on the two last bits such that either only one bit receives pheromone above $1/n$ or one bit receives pheromone $1-1/n$ and the other bit gets the rest.
To show that the pheromones raise appropriately, we aim at a larger gain in the best number of ones.
The probability of constructing at least $j+2$ ones with any of the above-described vectors is at least
$(1-1/n)^{j-2} \cdot \binom{n-j+2}{3} \cdot 1/n^3 \cdot (1-1/n)^{n-j} \ge 1/(48e) = \Omega((n-j)/n)$.

Using the same choice $t' := \lceil 2/\rho \rceil$ as above, Lemma~\ref{lem:MultiStepPheromoneDrift} yields $f(\tau^{t+t'}) \ge \alpha_{j-1} + 3 (1-(1-\rho)^{t'}) \ge \alpha_{j-1} + 2 \ge v(a_j)$.

Summing up the expected times for all phases yields a bound of
\[
\mathord{O}\mathord{\left(\sum_{i=0}^{n-1} \frac{n}{n-i} + n(t+t')\right)} = \mathord{O}\mathord{\left(n \log n + n/\rho\right)}.
\]
\end{proof}

\subsection{BinVal}
\label{sec:ACOOnBinVal}

The function \BV has similar properties as the well-known function $\LO(x) := \sum_{i=1}^n \prod_{j=1}^i x_j$ that counts the number of leading ones.
For both \MMAS and \MMASs the leading ones in $x^*$ can never be lost as setting one of these bits to 0 will definitely result in a worse solution.
This implies for both algorithms that the pheromones on the first $\LO(x^*)$ bits will strictly increase over time, until the upper pheromone border is reached.

In~\cite{DBLP:journals/swarm/NeumannSW09} the following upper bound for \LO was shown.
\begin{theorem}[\cite{DBLP:journals/swarm/NeumannSW09}]
\label{the:lo-upper}
  The expected optimization time of \MMAS and \MMASs on \LO is bounded
  by $O(n^2 + n/\rho)$ and $\mathord{O}\mathord{\left(n^2 \cdot
      (1/\rho)^\varepsilon + \frac{n/\rho}{\log(1/\rho)}\right)}$ for
  every constant $\varepsilon > 0$.
\end{theorem}

The basic proof idea is that after an average waiting time of $\ell$ iterations the probability of rediscovering the leading ones in $x^*$ is at least $\Omega(e^{-5/(\ell\rho)})$. Plugging in appropriate values for $\ell$ then gives the claimed bounds.

\begin{definition}
For $\ell \in \N$ and a sequence of bits $x_1, \dots, x_i$ ordered with respect to increasing pheromones we say that these bits form an $(i, \ell)$-layer if for all $1 \le j \le i$
\[
\tau_j \ge \min(1-1/n, 1-(1-\rho)^{j\ell}).
\]
\end{definition}

With such a layering the considered bits can be rediscovered easily, depending on the value of~$\ell$.
The following lemma was implicitly shown in~\cite[proof of Theorem~6]{DBLP:journals/swarm/NeumannSW09}. 
\begin{lemma}
\label{lem:rediscover-layer}
The probability that in an $(i, \ell)$-layer all $i$ bits defining the layer are set to 1 in an ant solution is $\Omega(e^{-5/(\ell\rho)})$.
\end{lemma}

Assume we have $k = \LO(x^*)$ and the pheromones form a $(k, \ell)$-layer. Using Lemma~\ref{lem:rediscover-layer} and the fact that a new leading one is added with probability at least $1/n$, the expected waiting time until we have $\LO(x^*) \ge k+1$ and a $(k+1, \ell)$-layer of pheromones is at most $O(n \cdot e^{5/(\ell\rho)} + \ell)$. As this is necessary at most $n$ times, this gives us an upper bound on the expected optimization time.
Plugging in $\ell = \lceil5/\rho\rceil$ and $\ell = \lceil5/(\varepsilon \rho \ln(1/\rho))\rceil$ yields the same upper bounds for \BV as we had in Theorem~\ref{the:lo-upper} for \LO.
\begin{theorem}
\label{the:binval-upper}
  The expected optimization time of \MMAS and \MMASs on \BV is bounded
  by $O(n^2 + n/\rho)$ and $\mathord{O}\mathord{\left(n^2 \cdot
      (1/\rho)^\varepsilon + \frac{n/\rho}{\log(1/\rho)}\right)}$ for
  every constant $\varepsilon > 0$.
\end{theorem}

These two bounds show that the second term ``$+n/\rho$'' in the first bound---that also appeared in the upper bound for \ONEMAX---can be lowered, at the expense of an increase in the first term. It is an interesting open question whether for all linear functions when $\rho$ is very small the runtime is $o(n/\rho)$, \ie, sublinear in $1/\rho$ for fixed $n$. The relation between the runtime and $\rho$ is further discussed from an experimental perspective in the next section.

\section{Experiments}
\label{sec:exp}
\ignore{
\remark{this section has to be updated, depending on what are the final theoretical results from chapter 4}

In previous sections we have shown a runtime bound of $O((n^3 \log n) / \rho)$. However, for $\rho = 1$, the behavior of \MMASs equals that of the \oneoneeas, for which a runtime bound of $O(n \log n)$ is known. Hence, we conducted experiments for variable $n$ and $\rho$, which indicate that the expected optimization time of the \MMAS on linear functions is $O((n \log n) / \rho)$.
The experimental results are shown in Figures~\ref{fig:experiments-mmass} and~\ref{fig:experiments-mmas}.
}

In this section, we investigate the behavior of our algorithms using experimental studies. Our goal is to examine the effect of the pheromone update strength as well as the impact of the weights of the linear function that should be optimized.

Related to our experiments are those in \cite{doerr2008, DBLP:journals/swarm/NeumannSW09}. The authors of the first article concentrate their analyses of \oneant and \MMAS on \ONEMAX, \LO, and random linear functions. A single $n$ for each function is used, and a number of theory-guided indicators monitors the algorithms' progress at every time step, in order to measure the algorithms' progress within individual runs. In the second article, the runtime of \MMAS and \MMASs is investigated on \ONEMAX and other functions, for two values of $n$ and a wide range of values of $\rho$.


\begin{figure*}[!htp]
\begin{center}
\hspace*{-20mm}
\begin{tabular}{cc}
\includegraphics[width=8cm]{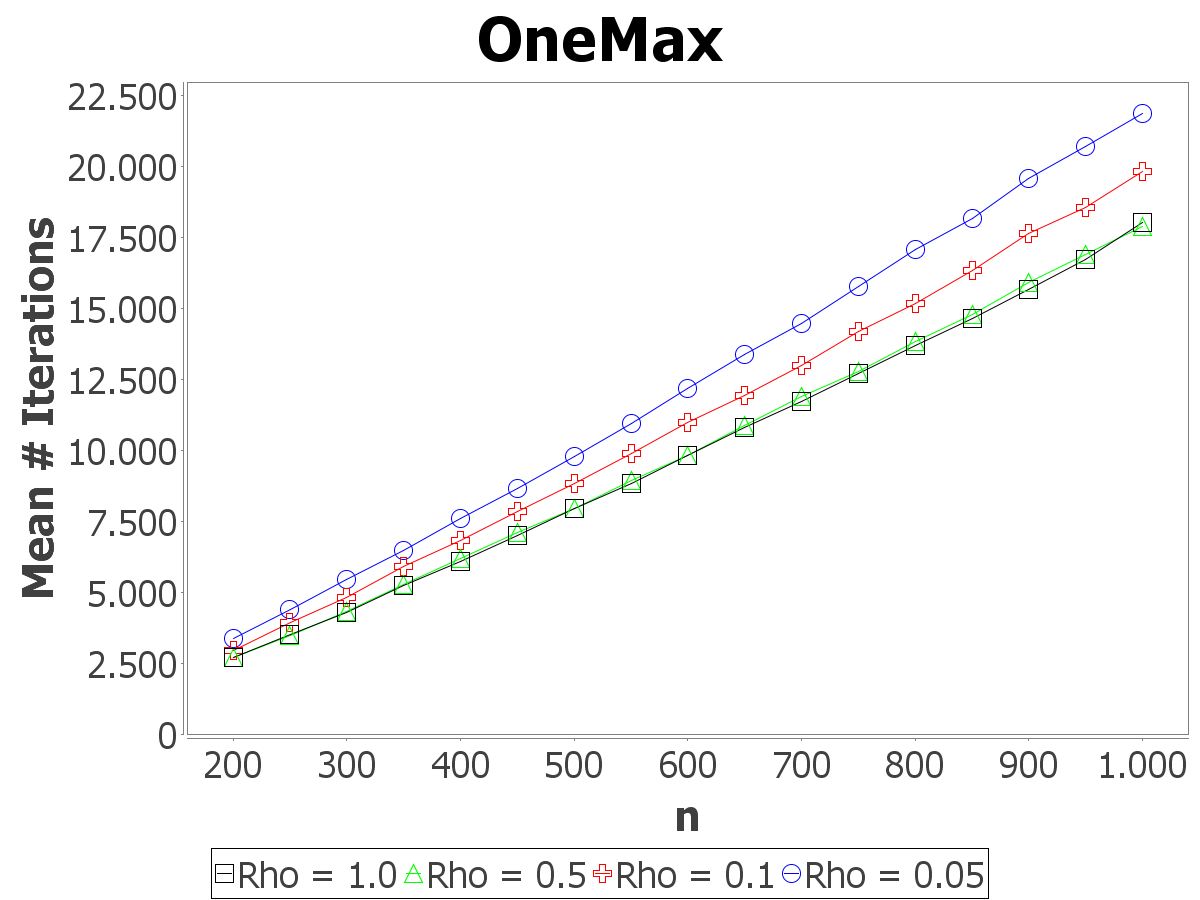} & \includegraphics[width=8cm]{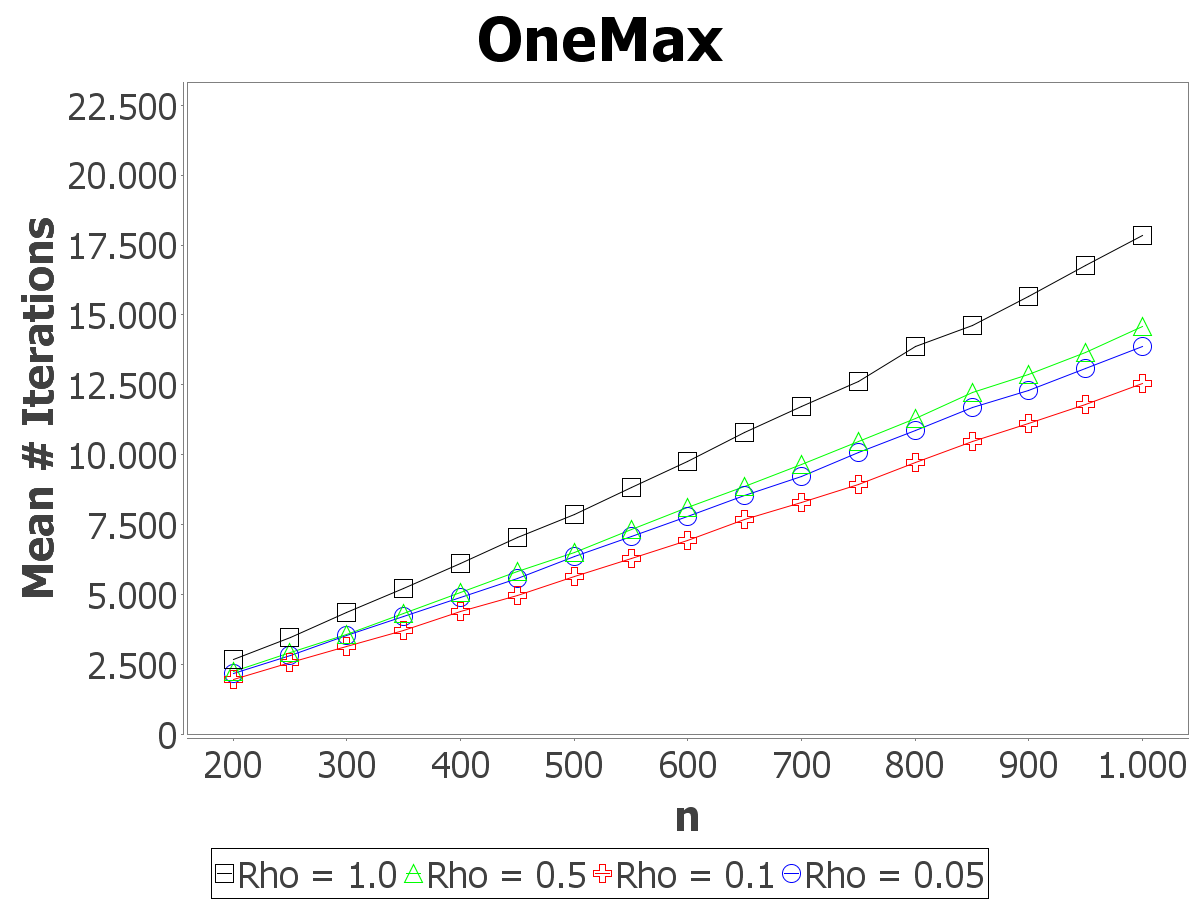} \\
\includegraphics[width=8cm]{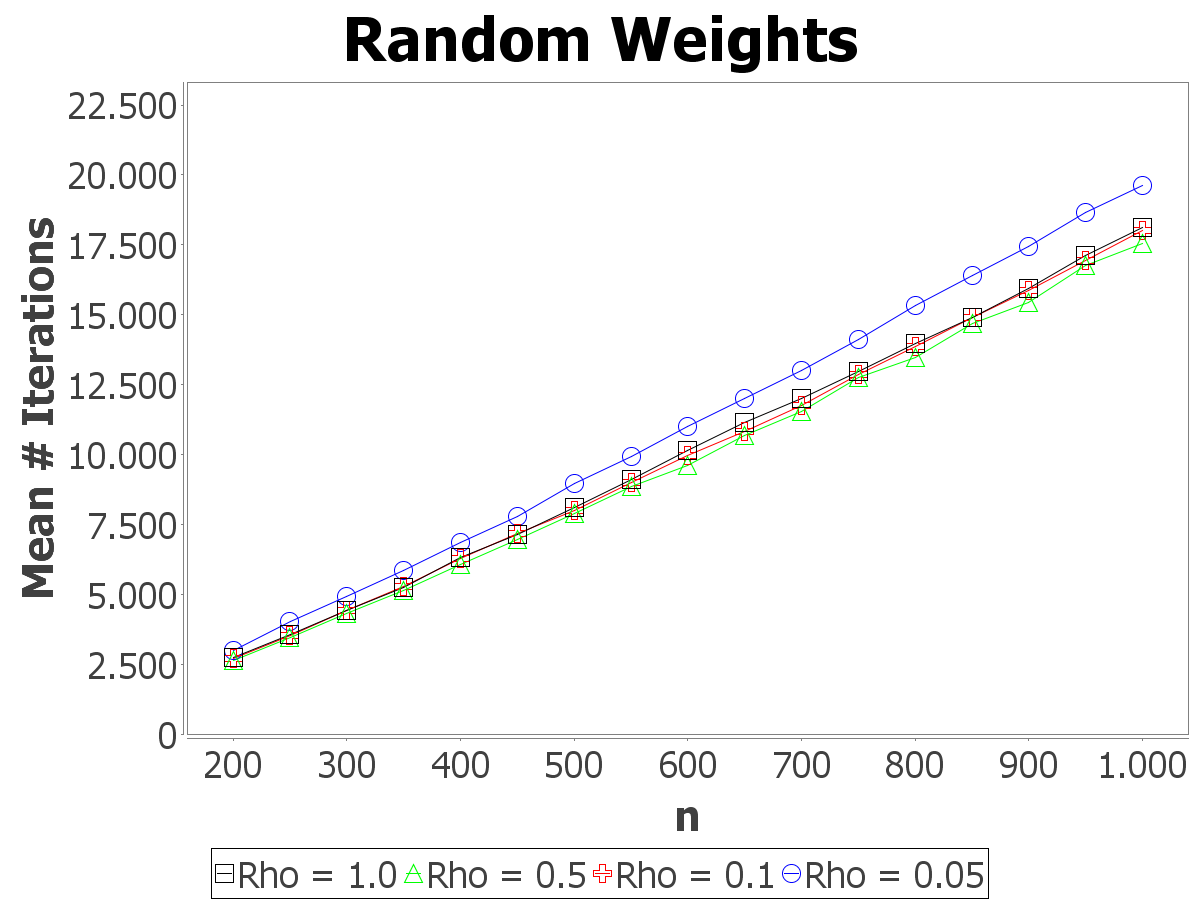} & \includegraphics[width=8cm]{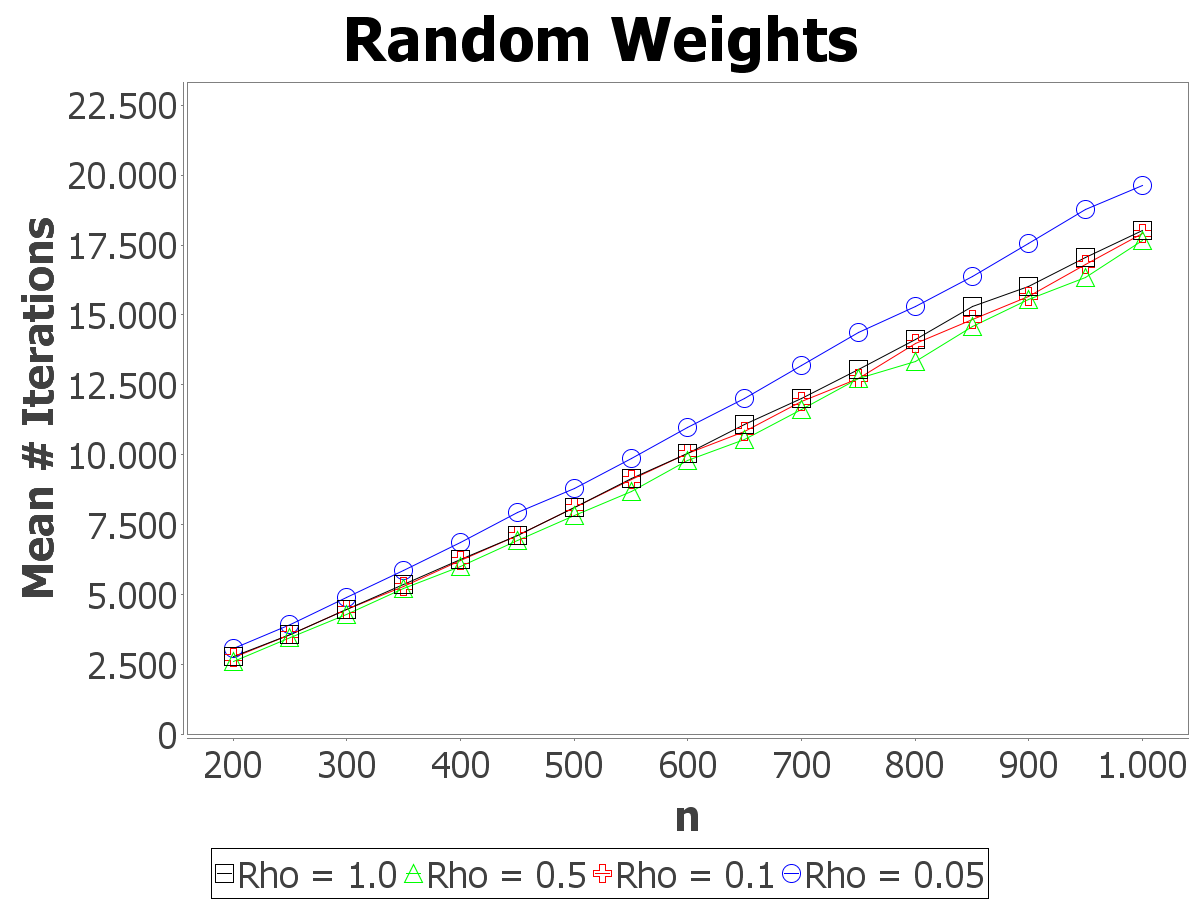} \\
\includegraphics[width=8cm]{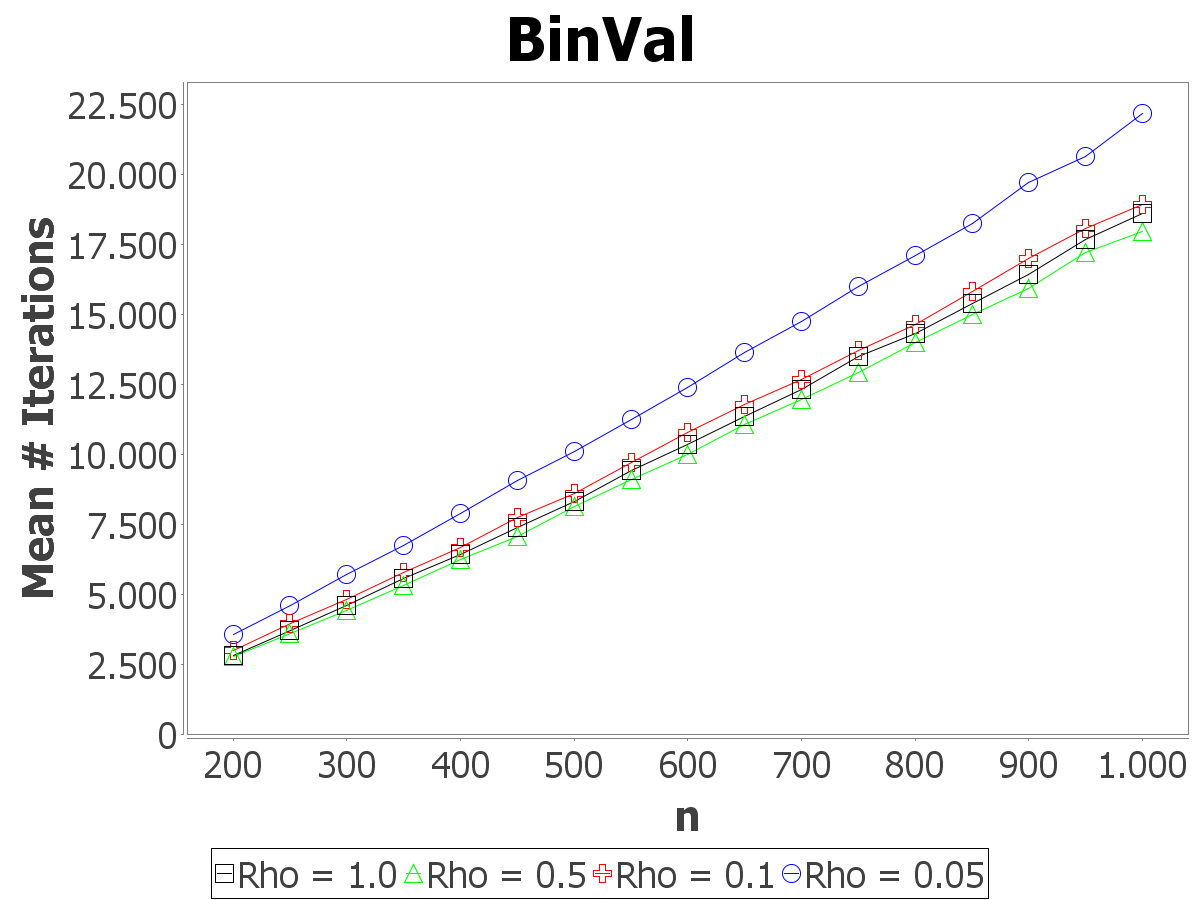} & \includegraphics[width=8cm]{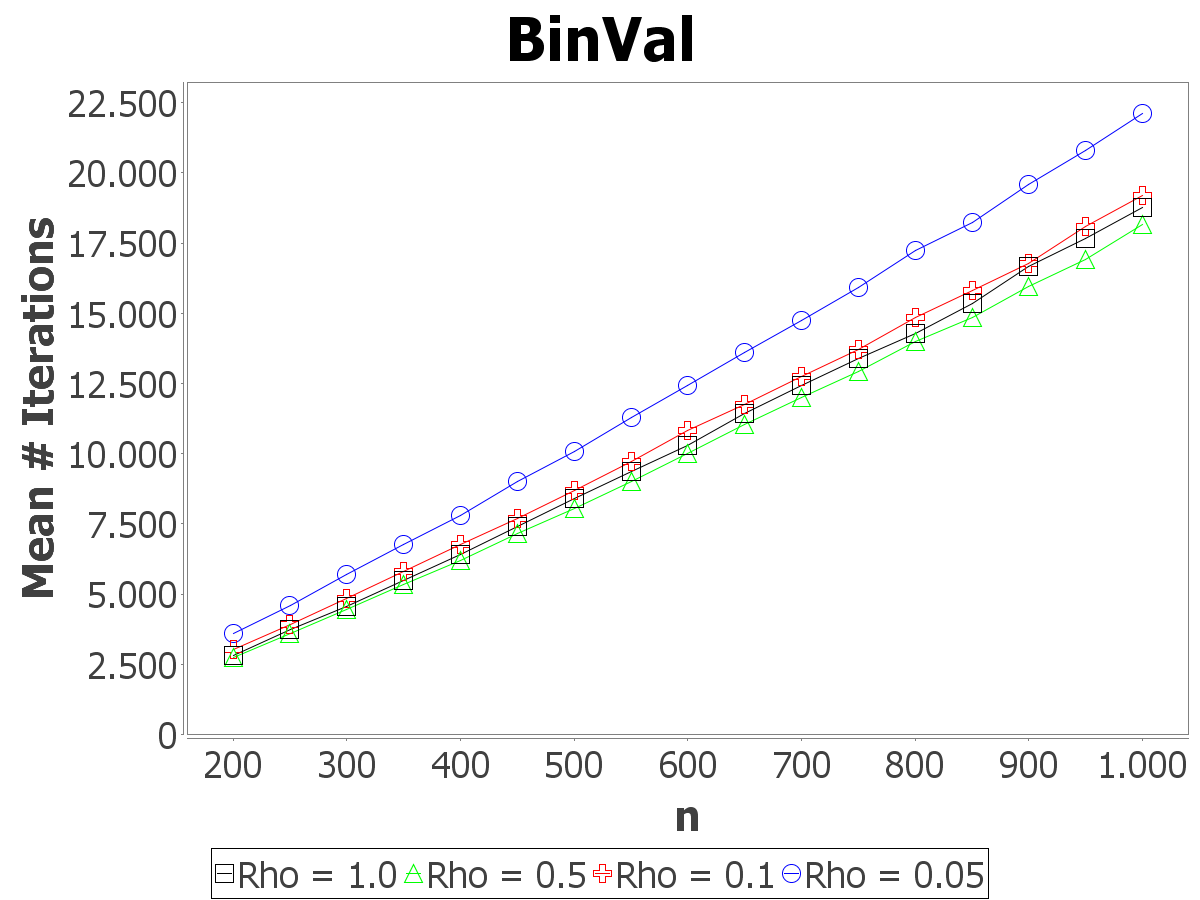}
\end{tabular}
\end{center}

\caption{Runtime of \MMASs (left column) and \MMAS (right column).}
\label{fig:runtimeN}
\end{figure*}

First, we investigate the runtime behavior of our algorithms for different settings of $\rho$ (see Figure~\ref{fig:runtimeN}).
We investigate \ONEMAX, \BV, and random functions where the weights are chosen uniformly at random from the interval $]0,1]$.
For our investigations, we consider problems of size $n=200, 250, \ldots, 1000$.
Each fixed value of $n$ is run for $\rho \in \{1.0, 0.5, 0.1, 0.05 \}$
and the results are averaged over 1000 runs.
Remember that, for $\rho=1.0$, \MMASs is equivalent to the \oneoneeas, and \MMAS is equivalent to the \oneoneea (see \cite{JWplateau} for the definition of the evolutionary algorithms).

One general observation is that the performance of \MMASs on random linear functions and \BV is practically identical to that of \MMAS. This was expected, as several bit positions have to have the same associated weights in order for \MMAS to benefit from its weakened acceptance condition. Furthermore, we notice that \ONEMAX is not the simplest linear function for \MMASs to optimize. In fact, for certain values of $\rho$, \ONEMAX is as difficult to optimize as \BV.

For all experiments, the performance of \MMASs with $\rho=1.0$ is very close to that of \MMAS with $\rho=1.0$. However, with different values of $\rho$, several performance differences are observed. For example, \MMASs with $\rho=0.5$ and $\rho=0.1$ optimizes random linear functions faster than \MMASs with $\rho=1.0$, which is on the other hand the fastest setup for \ONEMAX. Furthermore, the performance of \MMAS increased significantly with values of $\rho<1.0$, \eg, \MMAS with $\rho=0.1$ is 30\% faster than \MMAS with $\rho=1.0$.
Another general observation is that \MMASs performs better on random linear functions, than on \ONEMAX, \eg for $n=1000$ and $\rho=0.1$ the runtime decreases by roughly 10\%.

In the following, we give an explanation for this behavior.
During the optimization process, it is possible to replace a lightweight 1 at bit $i$ (\ie, a bit with a relatively small associated weight $w_i$) with a heavyweight 0 at bit $j$ (\ie, a bit with a relatively large associated weight $w_j$). Afterwards, during the freezing process, the probability for sampling again the lightweight 1 at bit $i$ (whose associated $\tau_i$ is in the process of being reduced to $\tmin$) is relatively high.
Unlike in the case of \ONEMAX, it is indeed possible for \MMASs to collect heavyweight 1s in between, and the ``knowledge'' of the lightweight 1s is available for a certain time, stored as a linear combination of the pheromones' values.
This effect occurs in the phase of adjusting the pheromones, not when the pheromone values are saturated. Otherwise, the effect could be observed for the \oneoneea and the \oneoneeas as well.

\begin{figure*}[!htp]
\begin{center}
\hspace*{-20mm}
\begin{tabular}{cc}
\includegraphics[width=8cm]{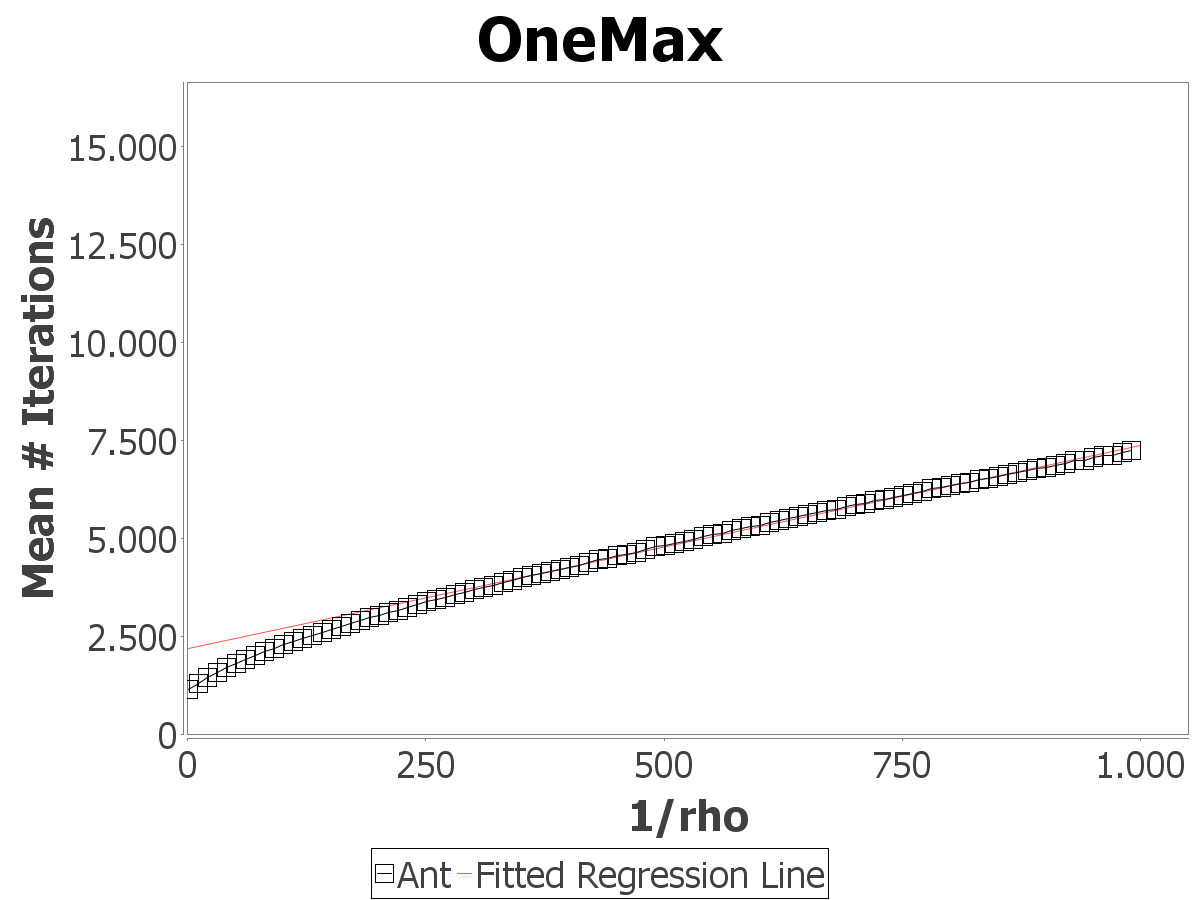} & \includegraphics[width=8cm]{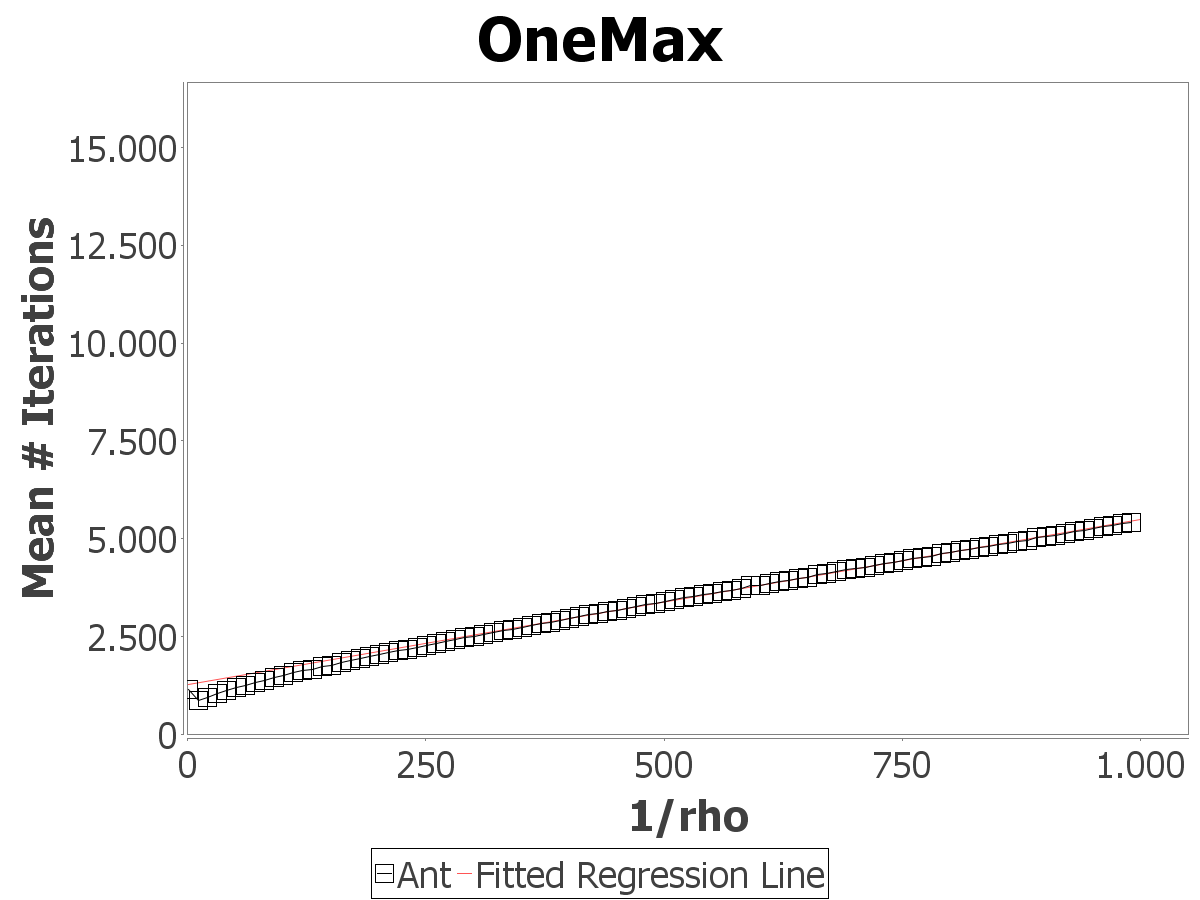} \\
\includegraphics[width=8cm]{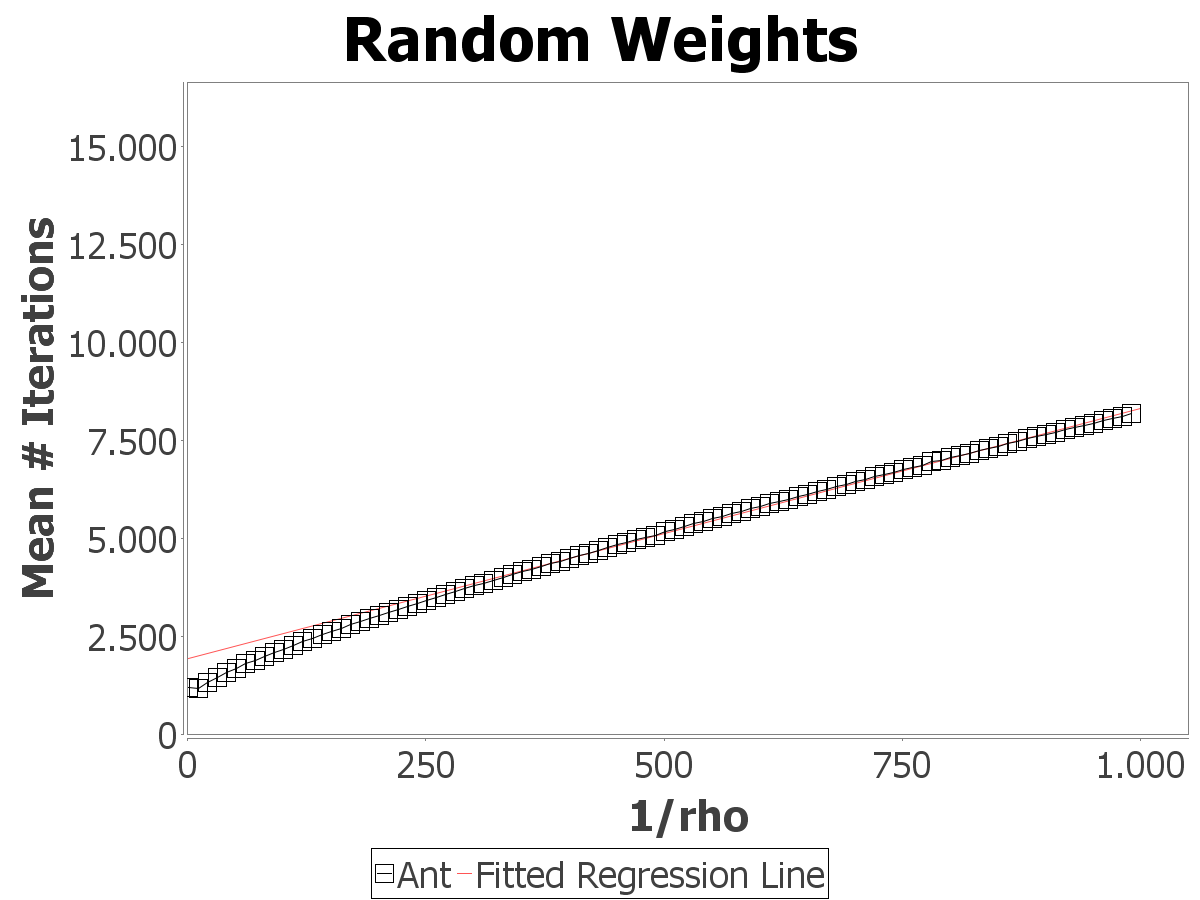} & \includegraphics[width=8cm]{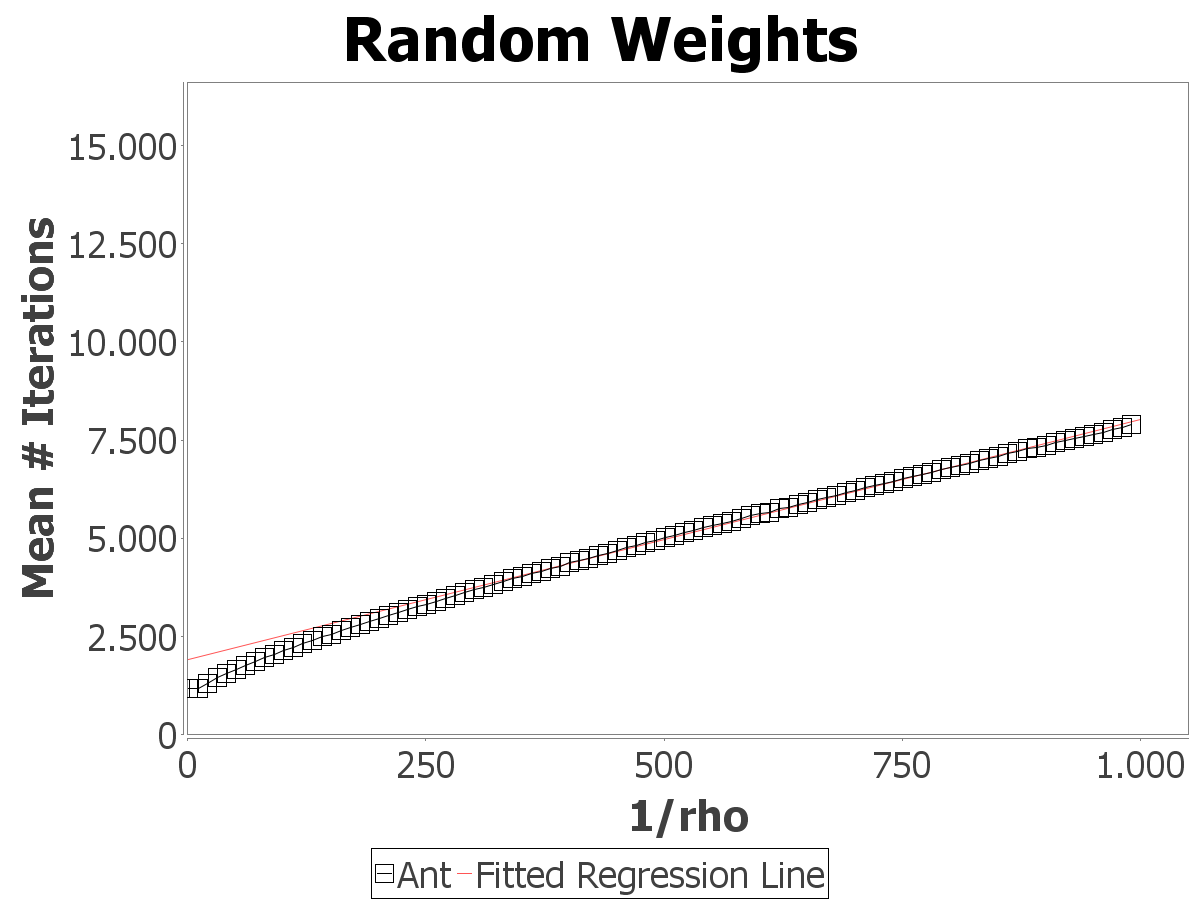} \\
\includegraphics[width=8cm]{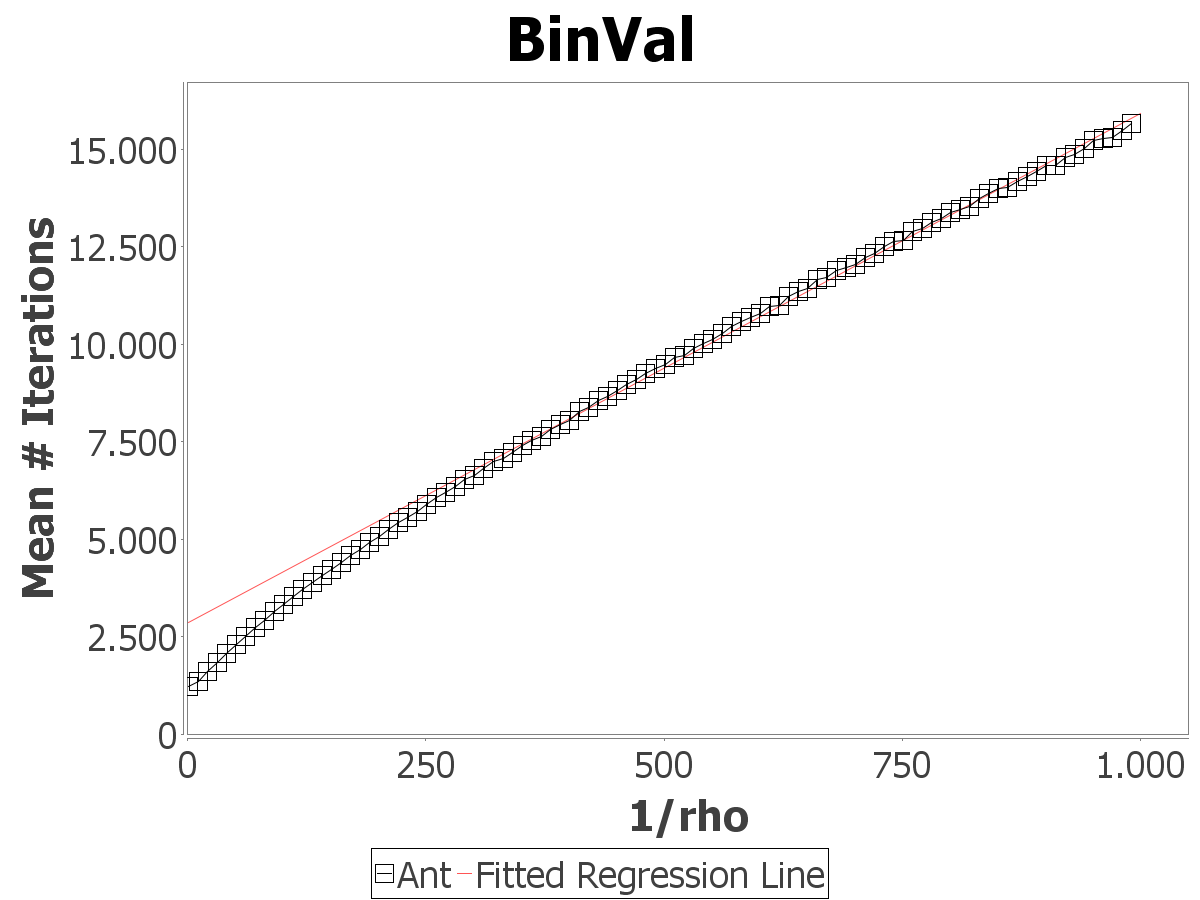} & \includegraphics[width=8cm]{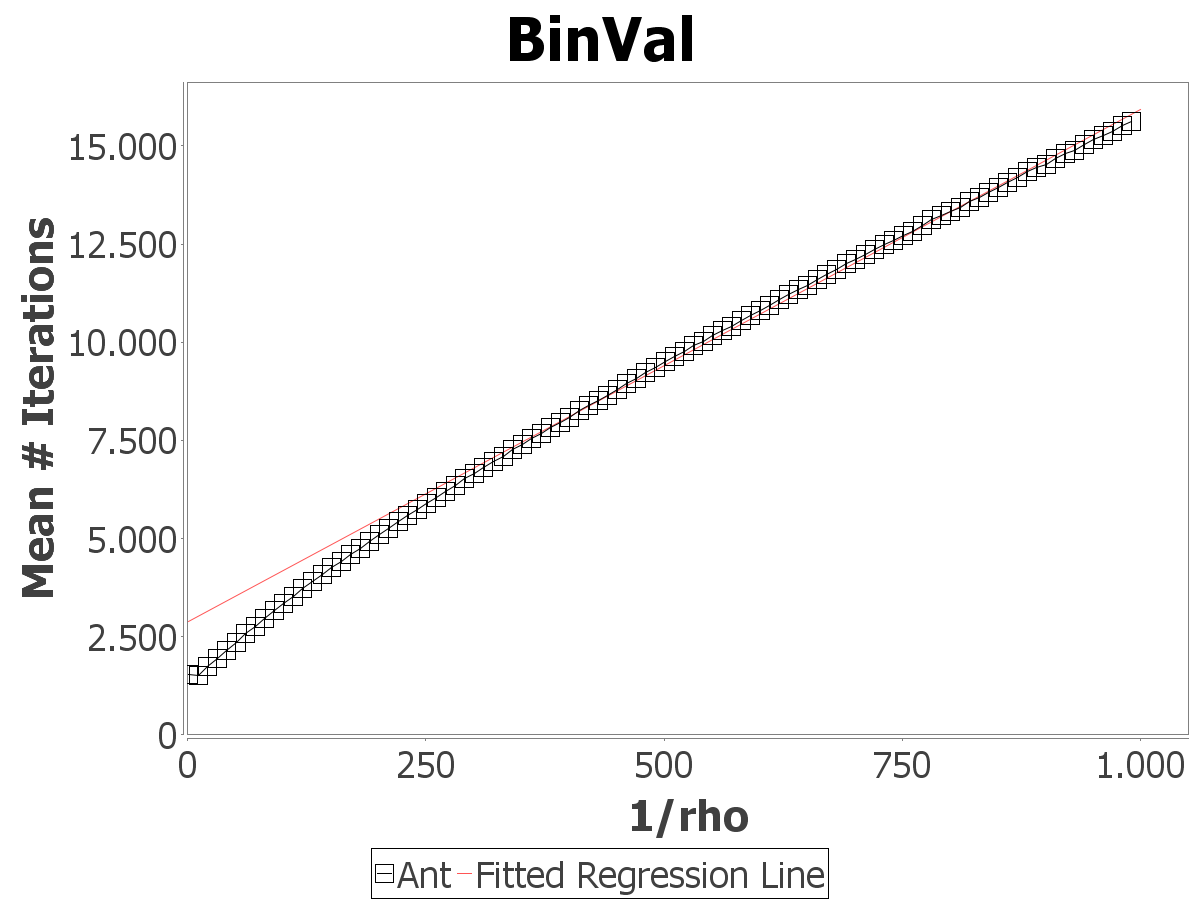}
\end{tabular}
\end{center}
\caption{Impact of pheromone evaporation factor in \MMASs (left column) and \MMAS (right column).}
\label{fig:runtimeRho}
\end{figure*}

We have already seen that the choice of $\rho$ may have a high impact on the optimization time.
The runtime bounds given in this paper increase with decreasing $\rho$.
In the following, we want to investigate the impact of $\rho$ closer by conducting experimental studies. We study the effect of $\rho$ by fixing $n=100$ and varying $\rho = 1/x$ with $x=1, 11, \ldots, 1001$. The results are shown in Figure~\ref{fig:runtimeRho} and are averaged over 10.000 runs.
The effect that small values of $\rho$ can improve the performance on \ONEMAX is observable again. The fitted linear regression lines, which are based on the mean iterations for $1/\rho \in \left(500,1000\right]$, support our claim that the runtime grows at most linear with $1/\rho$ for a fixed value of $n$.
In fact, the fitted lines indicate that the growth of the average runtime is very close to a linear function in $1/\rho$. The real curves appear to be slightly concave, which corresponds to a sublinear growth. However, the observable  effects are too small to allow for general conclusions.

\section{Conclusions and Future Work}

The rigorous analysis of ACO algorithms is a challenging task as these algorithms are of a high stochastic nature. Understanding the pheromone update process and the information that is stored in pheromone during the optimization run plays a key role to increase their theoretical foundations.

We have presented improved upper bounds for the performance of ACO on the class of linear pseudo-Boolean functions.
The general upper bound of $O((n^3 \log n)/\rho)$ from Theorem~\ref{the:general-upper-bound-linear} applies to all linear functions, but in the light of the smaller upper bounds for \ONEMAX and \BV we believe that this bound is still far from optimal. Stronger arguments are needed in order to arrive at a stronger result.

We also have developed novel methods for analyzing ACO algorithms without relying on pheromones freezing at pheromone borders.
Fitness-level arguments on a pheromone level have revealed one possible way of reasoning.
For \ONEMAX this approach, in combination with results from~\cite{Gleser:j:75}, has led to a
bound of $O(n \log n + n/\rho)$, both for \MMAS and \MMASs. This is a major improvement to the previous best known bounds $O((n^3 \log n)/\rho)$ for \MMAS and $O((n \log n)/\rho)$ for \MMASs and it finally closes the gap of size $n^2$ between the upper bounds for \MMAS and \MMASs.
We conjecture that our improved bound holds for all linear functions, but this is still a challenging open problem.

The experimental results have revealed that a slow adaption of pheromone is beneficial for \MMAS on \ONEMAX as \MMAS was faster than the \oneoneea for all investigated $\rho$-values not larger than $0.5$. We also argued why \MMASs is faster on random-weight linear functions than on \ONEMAX. The experiments also gave more detailed insights into the impact of the evaporation factor on the average runtime.

We conclude with open questions and tasks for future work:
\begin{enumerate}
\item Do \MMAS and \MMASs optimize all linear pseudo-Boolean functions in expected time $O(n \log n + n/\rho)$?
\item Analyze ACO for combinatorial problems like minimum spanning trees and the TSP in settings with slow pheromone adaptation.
\end{enumerate}

\end{document}